\documentclass[11pt]{article}
\usepackage{latexsym}
\usepackage{graphicx}
\usepackage{caption}
\usepackage{subcaption}
\usepackage{macros}

\usepackage{amsmath}
\numberwithin{equation}{section}
\usepackage{amsthm}
\usepackage{amssymb}

\usepackage{macros}
\definecolor{darkblue}{rgb}{0.0,0,0.5}
\usepackage[colorlinks=true,allcolors=darkblue]{hyperref}
\usepackage[top=2.54cm,left=2.54cm,right=2.54cm,bottom=2.54cm]{geometry}

\renewcommand{\L}{\mathcal{L}}
\newcommand{\cont}{{\rm cont}}

\begin{document}

\title{On the Limitation of Kernel Dependence Maximization for Feature Selection} 

\author{
Keli Liu\thanks{Company A},~~
Feng Ruan\thanks{Department of Statistics and Data Science, Northwestern University}}
\maketitle
\maketitle

\begin{abstract}
 A simple and intuitive method for feature selection consists of choosing the feature subset that maximizes a nonparametric measure of dependence between the response and the features. A popular proposal from the literature uses the Hilbert-Schmidt Independence Criterion (HSIC) as the nonparametric dependence measure. The rationale behind this approach to feature selection is that important features will exhibit a high dependence with the response and their inclusion in the set of selected features will increase the HSIC. Through counterexamples, we demonstrate that this rationale is flawed and that feature selection via HSIC maximization can miss critical features.

\end{abstract}

\newcommand{\HSIC}{{\rm HSIC}}

\section{Introduction} 
\indent\indent 
In statistics and machine learning, feature selection aims to isolate 
a subset of features that fully captures the relationships between the features
and response variables. 
This critical step enhances model interpretability, reduces the computational burden
of downstream tasks and improves the model's generalization capabilities~\cite{HastieTiFr09}.

Nonparametric dependence measures are key tools for feature selection. Specifically, a 
nonparametric dependence measure between a response variable and a feature subset
quantifies the strength of dependence without assuming any parametric model for their relationship. 
Common examples include mutual information~\cite{Cover99}, 
distance covariance \cite{szekely2009brownian}, and kernel-based dependence 
measures such as the Hilbert-Schmidt Independence Criterion (HSIC)~\cite{GrettonBoSmSc05}. 
The HSIC, for a wide range of kernels, recognizes all modes of
dependence between 
variables, not just linear dependence, and large values of HSIC correspond to ``more dependence" 
(for the definition and the scope of HSIC we consider, see 
Section~\ref{sec:problem-setup-and-notation}). 
The literature has proposed selecting features by maximizing the HSIC~\cite{SongSmGrBoBe07, SongSmGrBeBo12}. 
Given features $X \in \R^p$ and response variable $Y \in \R$, at the population level, this approach 
selects a feature subset $X_{S_*}$, where $S_*$ is defined by
\begin{equation}
\label{eqn:variable-selection-general}
S_* = \argmax_{S: S \subseteq \{1, 2, \ldots, p\}} \HSIC(X_S; Y).
\end{equation}
The measure $\HSIC(X_S; Y)$
quantifies the nonparametric dependence between the subset of features $X_S$ and the 
response variable $Y$, with $X_S \in \R^{|S|}$ denoting the subvector formed by $X$ by including only the 
coordinates in the set $S$ (for a formal definition of $\HSIC(X_S; Y)$, see 
Section~\ref{sec:main-results}). 
The intuition behind this approach is clear: the subset of
features, $X_{S_*}$, exhibiting maximum dependence with the response $Y$ (as measured by 
$\HSIC$) should hopefully include all features necessary for explaining $Y$. 


Despite its widespread use~\cite{SongBeBoGrSm07, GeeithaTh18, WangDaLi21},
this paper reveals a critical problem in the use of HSIC for feature selection, an issue unaddressed in the existing literature. 
Through a counterexample, we show that HSIC maximization can fail to select
critical variables needed to fully explain the response, even at the population level. In our counterexample,
the set of selected variables $S_*$, which is defined by maximizing HSIC through
equation~\eqref{eqn:variable-selection-general}, 
does not guarantee full explanatory power: 
\begin{equation}
	\E[Y|X] \neq \E[Y|X_{S_*}]~~\text{and thus}~~\L(Y|X) \neq \L(Y|X_{S_*})
\end{equation}
where the notation $\E[\cdot \mid \cdot]$ and $\L(\cdot \mid \cdot)$ denote the conditional expectation and 
conditional distribution, respectively. The first inequality is in the $\mathcal{L}_2(\P)$ sense. 


Our counterexample shows that selecting features by maximizing HSIC can mistakenly 
exclude variables essential for full explaining the response variable $Y$. This result, as far as we are aware, enhances 
the current understanding of HSIC's statistical properties in feature selection (see Section~\ref{sec:prior-literature} for a more detailed literature review). 
In particular, our result helps elucidate the distinction between two approaches in the literature to feature selection using kernel dependence measures. 
The first approach, termed ``dependence maximization'', consists of finding some feature subset $S$ to 
maximize the kernel dependence between $X_S$ and the response $Y$ \cite{SongSmGrBoBe07}. 
The second approach, termed
``conditional dependence minimization'', consists of finding some feature subset $S$ such that the kernel \emph{conditional}
dependence between the response $Y$ and the feature vector $X$ is \emph{minimized} given $X_S$ \cite{FukumizuBaJo09}. 
While the first approach can be computationally cheaper, it has been 
\emph{unrecognized} that this first approach does not guarantee selecting all variables necessary 
for explaining the response
(see Section~\ref{sec:prior-literature} for more discussion of these two approaches).

\subsection{Problem Setup: Definition, Notation and Assumption} 
\label{sec:problem-setup-and-notation}
The Hilbert-Schmidt Independence Criterion (HSIC), introduced initially in~\cite{GrettonBoSmSc05}, 
is a 
measure quantifying the dependence between two random variables, $X$, and $Y$, using 
positive (semi)definite kernels. A positive semidefinite kernel $k$, simply called a kernel, on a 
set $\mathcal{X}$ is a symmetric function 
$k: \mathcal{X} \times \mathcal{X} \to \R$ that requires the sum
$\sum_{i=1}^n \sum_{j=1}^n c_i c_j k(x_i, x_j)$ to be nonnegative for any $n$ distinct 
points $x_1, x_2, \ldots, x_n \in \mathcal{X}$ and 
any real numbers $\{c_1, c_2, \ldots, c_n\}$. 
A kernel $k$ is called a positive definite kernel when this sum equals $0$ if and only 
if all $c_i$ are zero.  
 
We are ready to give a formal definition of HSIC, following~\cite{GrettonBoSmSc05, SongSmGrBeBo12}. 

\begin{definition}[HSIC]
\label{definition:HSIC}
Given a joint probability measure $(X, Y) \sim \Q$ over the 
product space $\mathcal{X} \times \mathcal{Y}$, 
along with two positive semidefinite kernels $k_X$, $k_Y$ on the sets $\mathcal{X}$ and $\mathcal{Y}$ 
respectively, the Hilbert-Schmidt Independence Criterion (HSIC) is defined as 
\begin{equation}
\label{eqn:HSIC-definition}
\begin{split} 
	\HSIC(X, Y) &= \E[k_X(X, X') \cdot  k_Y(Y, Y')] + \E[k_X(X, X')] \cdot \E[k_Y(Y, Y')] \\
		&~~~~ - 2\E \left[\E[k_X(X, X')|X'] \cdot \E[k_Y(Y, Y')|Y']\right].
\end{split} 
\end{equation}
In the above, all expectations are taken over independent copies $(X, Y), (X', Y') \sim \Q$.
\end{definition} 

The HSIC is always non-negative 
($\HSIC(X; Y) \ge 0$) for every probability distribution $\Q$ and
positive semidefinite kernels $k_X$ and $k_Y$~\cite{GrettonBoSmSc05, SongSmGrBeBo12}. 
Clearly, $\HSIC(X; Y) = 0$ whenever $X$ and $Y$ are independent. 
Additionally, when the kernels $k_X$ and $k_Y$ in the Definition~\ref{definition:HSIC} are 
characteristic kernels in the 
sense of~\cite[Section 2.2]{FukumizuGrSuSc07} or~\cite[Definition 6]{SriperumbudurFuLa11}, 
then $\HSIC(X; Y) = 0$ if and only if $X$ and $Y$ are independent (see~\cite[Theorem 4]{GrettonBoSmSc05}, and~\cite[Section 2.3]{SriperumbudurFuLa11} for the statement). Characteristic kernels enable HSIC to capture
all modes of dependence between $X$ and $Y$. Non-characteristic kernels can also be plugged into the definition of HSIC, equation~\eqref{eqn:HSIC-definition}, but they may fail to detect some forms of dependence between $X$ and $Y$ i.e. it is possible that $\HSIC(X; Y) = 0$ but $X$ and $Y$ are not independent.

Our counterexamples apply to a wide range of HSIC through various choices of kernels
$k_X$ and $k_Y$. Two requirements on the kernels $k_X$ and $k_Y$ are needed 
throughout the paper. The first requires that $k_X$ is a radially symmetric positive definite kernel.
\begin{assumption}
\label{assumption:kernel-representation}
The kernel $k_X$ obeys for every $x, x' \in \R^p$: 
\begin{equation*}
	k_X(x, x') = \phi_X(\norm{x-x'}_2^2).
\end{equation*}
where, for some finite nonnegative measure 
$\mu$ that is not concentrated at $0$, the following holds for every $z \ge 0$: 
\begin{equation*}
	\phi_X(z) = \int_0^\infty e^{-tz} \mu(dt).
\end{equation*}
\end{assumption} 
The integral representation of $\phi_X$ in Assumption~\ref{assumption:kernel-representation} is 
motivated by Schoenberg’s fundamental work~\cite{Schoenberg38}. Given a real-valued function 
$\phi_X$, for the mapping 
$(x, x') \mapsto \phi_X(\norm{x-x'}_2^2)$ to qualify as a positive definite 
kernel on $\R^m$ for every dimension 
$m \in \N$, the representation $\phi_X(z) = \int_0^\infty e^{-tz} \mu(dt)$ must hold for 
some nonnegative finite measure $\mu$ not concentrated at zero~\cite{Schoenberg38}. Two concrete examples satisfying 
Assumption~\ref{assumption:kernel-representation} are the Gaussian 
kernel where $\phi_X(z) = \exp(-z)$, and the Laplace kernel where $\phi_X(z) = \exp(-\sqrt{z})$.

Assumption~\ref{assumption:kernel-representation} 
ensures that the mapping  
$(x, x') \mapsto \phi_X(\norm{x-x'}_2^2)$ is a kernel on $\R^m$ for every dimension $m \in \N$. 
This property is useful when working with $\HSIC(X_S; Y)$
because the dimension of the feature vector $X_S$ is varying with the size of $S$ and may 
not equal to $p$. 
As a side remark, the kernel 
$(x, x') \mapsto \phi_X(\norm{x-x'}_2^2)$ is a characteristic
kernel on $\R^m$ for every dimension $m \in \N$ under 
Assumption~\ref{assumption:kernel-representation}~\cite[Proposition 5]{SriperumbudurFuLa11}.

Our second requirement is that the kernel $k_Y$ is either a function of $|y - y'|$ or a function of $yy'$. 
\begin{assumption}
\label{assumption:kernel-representation-Y}
The kernel $k_Y$ satisfies one of the following conditions:
\begin{itemize}
\item $k_Y(y, y') = \phi_Y(|y-y'|)$ with $\phi_Y(0) > \phi_Y(1)$ where $\phi_Y$ is a real-valued function.
\item $k_Y(y, y') = \phi_Y(yy')$ with $\phi_Y(1) > \phi_Y(-1)$ where $\phi_Y$ is a real-valued function.
\end{itemize} 
\end{assumption} 
For $k_Y(y, y') = \phi_Y(|y-y'|)$ to be a positive semidefinite kernel, $\phi_Y(0) \ge \phi_Y(1)$ 
must hold, as $\phi_Y(0) - \phi_Y(1)=\sum_{i,j=1}^2 c_i c_j k_Y(y_i, y_j) \ge 0$ where 
$c_1 = 1, c_2 = -1, y_1 = 1, y_2 = 0$.
Likewise, for $k_Y(y, y') = \phi_Y(yy')$ to be a positive semidefinite kernel, $\phi_Y(1) \ge \phi_Y(-1)$ must be true. Thus the inequalities in Assumption~\ref{assumption:kernel-representation-Y} are only slightly stricter than the requirement that $k_Y$ is a positive semidefinite kernel. Assumption~\ref{assumption:kernel-representation-Y} is very mild since it includes 
both characteristic kernels such as the Gaussian kernel 
$k_Y(y, y') = \exp(-(y-y')^2)$ and Laplace kernel $k_Y(y, y') = \exp(-|y-y'|)$ as well as 
non-characteristic kernels such as the inner product kernel $k_Y(y, y') = yy'$. 

Our counterexamples apply to any kernel $k_X, k_Y$ obeying 
Assumptions~\ref{assumption:kernel-representation} and~\ref{assumption:kernel-representation-Y}. 
Thus, it applies to a wide range of $\HSIC$, including those nonparametric dependence measures utilizing
characteristic kernels $k_X$ and $k_Y$ to capture all types of dependence
(say, when $k_X$ and $k_Y$ are both Gaussian kernels). In these cases,  
$\HSIC(X; Y) = 0$ if and only if $X$ and $Y$ are independent.



\subsection{Main Results} 
\label{sec:main-results}
Throughout the remainder of the paper, we operate with two kernels 
$k_X, k_Y$ on $\R^p$ and $\R$ that meet Assumptions~\ref{assumption:kernel-representation} 
and~\ref{assumption:kernel-representation-Y}, respectively. Our main results 
(counterexamples) for HSIC are applicable to any such kernels.

Given the covariates $X \in \R^p$ and the response $Y \in \R$, with $(X, Y) \sim \P$, in this paper
we study the \emph{population} level statistical properties of the solution 
set $S_*$ of the following HSIC maximization procedure: 
\begin{equation}
\label{equation:MMD-set-maximizer} 
	S_*= \argmax_{S: S\subseteq \{1, 2, \ldots, p\}}  \HSIC(X_S; Y),
\end{equation}
where $\HSIC(X_S; Y)$ is defined through the kernels $k_{X; |S|}$ and $k_Y$. Here the
kernel $k_{X; |S|}$ on $\R^{|S|}$ is induced by the original kernel $k_X$ on $\R^p$ in the following way
(recall $\phi_X$ in the definition of $k_X$):
\begin{equation*}
	k_{X; |S|}(x, x') = \phi_X(\norm{x-x'}_2^2)~~\text{for $x, x' \in \R^{|S|}$}.
\end{equation*}
For transparency, following Definition~\ref{definition:HSIC} we can write down the definition: 
\begin{equation}
\label{eqn:HSIC-S}
\begin{split} 
	\HSIC(X_S, Y) &= \E[k_{X; |S|}(X_S, X_S') \cdot  k_Y(Y, Y')] + \E[k_{X; |S|}(X_S, X_S')] \cdot \E[k_Y(Y, Y')] \\
		&~~~~ - 2\E \left[\E[k_{X; |S|}(X_S, X_S')|X'] \cdot \E[k_Y(Y, Y')|Y']\right],
\end{split} 
\end{equation} 
where $(X, Y)$, $(X', Y')$ denote independent draws from $\P$.
This definition of $\HSIC(X_S, Y)$ is natural, and aligns with its proposed use for feature selection, e.g.,~\cite{SongBeBoGrSm07}. 

Our main result shows that the feature subset $X_{S_*}$, chosen by maximizing 
$\HSIC$ as in equation~\eqref{equation:MMD-set-maximizer}
can fail to capture the explanatory relationship between $X$ and $Y$. 

\begin{theorem}
\label{theorem:inconsistency-HSIC} 
Assume the dimension $p \ge 2$.
Given any kernel $k_X$ on $\R^p$ obeying Assumption~\ref{assumption:kernel-representation} and 
any kernel $k_Y$ on $\R$ obeying Assumption~\ref{assumption:kernel-representation-Y}, there exists a probability 
distribution $\P$ of $(X, Y)$ supported on $\R^p \times \R$
such that any maximizer $S_*$ defined via equation~\eqref{equation:MMD-set-maximizer}
fails to include all relevant features in the following sense: 
\begin{equation*}
\text{$\E[Y|X]\neq \E[Y|X_{S_*}]$}~~\text{and thus}~~\L(Y|X) \neq \L(Y|X_{S_*}).
\end{equation*} 
In the above, the first inequality is under the $\mathcal{L}_2(\P)$ sense. 
\end{theorem} 

Theorem~\ref{theorem:inconsistency-HSIC} cautions against relying on HSIC 
 maximization for feature selection.
The reader might wonder whether it is the \emph{combinatorial} structure in defining the maximizer 
$S_*$ that is responsible for the failure observed in Theorem~\ref{theorem:inconsistency-HSIC}. 
To clarify the nature of the problem, we explore a continuous version 
of the combinatorial objective in equation~\eqref{equation:MMD-set-maximizer}. 
Below, we use $\odot$ to represent coordinate-wise 
multiplication between vectors; for instance, given a vector $\beta \in \R^p$, and covariates $X \in \R^p$, 
the product $\beta \odot X$ yields a new vector in $\R^p$ with its $i$th coordinate $(\beta\odot X)_i = \beta_i X_i$.

Now we consider HSIC maximization through a continuous optimization: 
\begin{equation}
\label{equation:MMD-set-maximizer-c} 
	S_{*,\cont} = \supp(\beta_*)~~\text{where}~~\beta_* = \argmax_{\beta: \norm{\beta}_{\ell_q} \le r} 
		\HSIC(\beta \odot X; Y)
\end{equation}
where $\HSIC(\beta \odot X, Y)$  measures the dependence between the 
weighted feature $\beta \odot X \in \R^p$ and the response $Y \in \R$, utilizing the kernels $k_X$ and $k_Y$
on $\R^p$ and $\R$, respectively: 
\begin{equation}
\label{eqn:HSIC-S}
\begin{split} 
	\HSIC(\beta \odot X, Y) &= \E[k_{X}(\beta \odot X, \beta \odot X') \cdot  k_Y(Y, Y')] + 
		\E[k_{X}(\beta \odot X, \beta \odot X')] \cdot \E[k_Y(Y, Y')] \\
		&~~~~ - 2\E \left[\E[k_{X}(\beta \odot X, \beta \odot X')|X'] \cdot \E[k_Y(Y, Y')|Y']\right].
\end{split} 
\end{equation} 
In its definition~\eqref{equation:MMD-set-maximizer-c},  we first find the weight 
vector $\beta_* \in \R^p$ that maximizes the HSIC dependence 
measure between the response $Y$ and the weighted covariates $\beta \odot X$, and then 
use the support set of $\beta_*$ (the set of indices corresponding to
non-zero coordinates of $\beta_*$) to define the feature set $S_{*, \cont}$. 
Note that imposing a norm constraint on the weights $\beta$---$\norm{\beta}_{\ell_q} \le r$ where
$\norm{\cdot}_{\ell_q}$ denotes the $\ell_q$ norm for $q \in [1, \infty]$ and 
$r < \infty$---ensures the existence of a well-defined maximizer $\beta_*$.


\begin{theorem}
\label{theorem:inconsistency-HSIC-b}
Assume the dimension $p \ge 2$.
Given any kernel $k_X$ on $\R^p$ obeying Assumption~\ref{assumption:kernel-representation}, 
kernel $k_Y$ on $\R$ obeying Assumption~\ref{assumption:kernel-representation-Y}, and any 
$\ell_q$ norm $\norm{\cdot}_{\ell_q}$ where $q \in [1, \infty]$ and radius $r < \infty$, there exists a probability 
distribution $\P$ of $(X, Y)$ supported on $\R^p \times \R$ such that any maximizer $\beta_*$ and the 
corresponding set $S_{*,\cont}$ defined via equation~\eqref{equation:MMD-set-maximizer-c} 
fail to include all relevant features: 
\begin{equation*}
\text{$\E[Y|X]\neq \E[Y|X_{S_{*,\cont}}]$}~~\text{and thus}~~\L(Y|X) \neq \L(Y|X_{S_{*, \cont}}).
\end{equation*} 
In the above, the first inequality is in the $\mathcal{L}_2(\P)$ sense. 
\end{theorem} 

Theorem~\ref{theorem:inconsistency-HSIC} and Theorem~\ref{theorem:inconsistency-HSIC-b} 
show that care must be taken when using HSIC for feature selection. 
Notably, the requirement that the dimension $p \ge 2$ is essential for 
Theorem~\ref{theorem:inconsistency-HSIC} and Theorem~\ref{theorem:inconsistency-HSIC-b}
to hold. This necessity follows from a positive result on HSIC maximization when using characteristic kernels $k_X, k_Y$: 
the output feature set $X_{S_*}$ (or $X_{S_{*,\cont}}$) will be dependent with $Y$ unless $X$ and $Y$ are 
independent.

\begin{proposition}
\label{proposition:positive-result}
Suppose both kernels $k_X$ and $k_Y$ are characteristic on $\R^p$ and $\R$ respectively 
in the sense of~\cite[Definition 6]{SriperumbudurFuLa11}. Applying the discrete 
HSIC maximization approach~\eqref{equation:MMD-set-maximizer} and the continuous 
HSIC maximization approach~\eqref{equation:MMD-set-maximizer-c} (with $r > 0$)
will result in nonempty feature sets $S_*$ and $S_{*, \cont}$, 
with $X_{S_*}$ and $X_{S_{*,\cont}}$ dependent on $Y$  respectively, 
unless $X$ and $Y$ are independent.
\end{proposition} 

\begin{proof}
Suppose $X$ and $Y$ are not independent. Then $\HSIC(X; Y) > 0$ because 
both kernels $k_X$ and $k_Y$ are characteristic kernels. The result follows from 
basic properties of $\HSIC$, see e.g.,
\cite[Theorem 4]{GrettonBoSmSc05}, and~\cite[Section 2.3]{SriperumbudurFuLa11}.

If we use the discrete maximization approach~\eqref{equation:MMD-set-maximizer},
then $\HSIC(X_{S_*}; Y) \ge \HSIC(X; Y)> 0$. This implies that $X_{S_*}$ and $Y$ are dependent because 
by definition of $\HSIC$, $\HSIC(X_{S_*}; Y)=0$ if $X_{S_*}$ and $Y$ are independent. Similarly, 
if we use the continuous HSIC maximization approach~\eqref{equation:MMD-set-maximizer-c}, 
then $\HSIC(\beta_*\odot X, Y) > 0$ which then implies $X_{S_*, \cont}$ and $Y$ cannot be independent.
\end{proof} 

To conclude, while intuitively appealing, using HSIC maximization for feature selection may inadvertently 
exclude variables necessary for fully understanding the response $Y$. 

%
%

\subsection{Connections to Prior Literature} 
\label{sec:prior-literature}
In the literature, there are two primary approaches to employing kernel-based nonparametric 
(conditional) dependence measures between variables for feature selection.

The first approach, directly relevant to our work, is to maximize the dependence measure 
between the response variable $Y$ and the feature subset $X_S$. In this paper, we focus on 
Hilbert Schmidt Independence Criterion (HSIC), a popular kernel-based nonparametric
dependence measure~\cite{GrettonBoSmSc05}. Unlike many other dependence measures, 
HSIC allows easy estimation from finite samples with parametric 
$1/\sqrt{n}$ convergence rates where $n$ is the sample size~\cite{GrettonBoSmSc05}.
Since its introduction, HSIC has become a popular tool for feature selection. 
This HSIC maximization method, formulated as equation~\eqref{equation:MMD-set-maximizer}, 
was developed by \cite{SongSmGrBoBe07} and has been applied in diverse fields such as 
bioinformatics \cite{SongBeBoGrSm07, GeeithaTh18}. Subsequent works \cite{MasaeliDyFu10, SongSmGrBeBo12} 
proposed continuous relaxations of this formulation to facilitate computation. 
Our main findings, Theorem \ref{theorem:inconsistency-HSIC} and Theorem \ref{theorem:inconsistency-HSIC-b}, 
provide concrete counterexamples advising caution when selecting features via HSIC maximization, 
as one may miss features needed for fully understanding the response $Y$.

The second approach minimizes the kernel conditional dependence between $Y$ and $X$ given the feature subset 
$X_S$ \cite{FukumizuGrSuSc07, FukumizuBaJo09, ChenStWaJo17, ChenLiLiRu23}. Remarkably, this approach 
ensures that the selected feature subset $X_{S_*}$ contains all features necessary to fully explain the response $Y$, 
as guaranteed by the conditional independence $Y \perp X \mid X_{S_*}$. This theoretical guarantee is achieved under a 
nonparametric setup without assuming specific parametric distribution models for the data $(X, Y)$~\cite{FukumizuBaJo09}. 
However, computing the kernel conditional dependence measures required by this approach is more 
computationally demanding than calculating the HSIC used in the maximization approach due to the need for kernel 
matrix inversions~\cite{FukumizuGrSuSc07}.

Notably, our main findings elucidate a distinction between these two kernel-based feature selection strategies. The maximization 
approach, which we study in this paper, does not enjoy the same conditional independence guarantee as the minimization approach. 
Our counterexamples in Theorem \ref{theorem:inconsistency-HSIC} and Theorem \ref{theorem:inconsistency-HSIC-b} demonstrate 
that the HSIC maximization approach can fail to select critical variables needed to fully explain the response $Y$, highlighting a 
tradeoff between computational effort and the ability to recover all important variables.

To clarify, our research focuses on maximizing HSIC for feature selection as described in equation~\eqref{equation:MMD-set-maximizer}. 
It is important to note that there are alternative methods of employing HSIC in feature selection, such as HSIC-Lasso
described in~\cite{YamadaJiSiXiSu14}, which involves selecting features through a least square regression that applies $\ell_1$ penalties
to the feature weights. HSIC-Lasso, though sharing a commonality in nomenclature with HSIC, is fundamentally a regression method, which  
differs from the dependence maximization strategy we explore here in equation~\eqref{equation:MMD-set-maximizer}. 
Our counterexamples do not extend to HSIC-Lasso.

Finally, consider the recent developments in feature selection for binary labels $Y \in \{\pm 1\}$, where a feature subset $S$ is 
selected by maximizing the Maximum Mean Discrepancy (MMD) between the conditional distributions $\mathcal{L}(X_S \mid Y=1)$
and $\mathcal{L}(X_S \mid Y=-1)$ over $S$ \cite{WangDeXi23, MitsuzawaKaBoGrPa23}. Given the established
connections between MMD and HSIC in the literature, e.g.,
\cite[Theorem 3]{SongSmGrBoBe07}, our counterexamples in Theorem~\ref{theorem:inconsistency-HSIC} and 
Theorem~\ref{theorem:inconsistency-HSIC-b}---developed under conditions with binary response $Y$
(see the proofs in Section~\ref{sec:proofs})---have further implications for these 
MMD-based feature selection methods. We intend to report these implications and discuss methodological developments 
in a follow-up paper.

\section{Proofs}
\label{sec:proofs}
In this section, we prove Theorem~\ref{theorem:inconsistency-HSIC} and 
Theorem~\ref{theorem:inconsistency-HSIC-b}. Throughout, for a vector 
$X \in \R^p$, we will denote the coordinates of $X$ under the standard Euclidean 
basis as $X_1, X_2, \ldots, X_p$. This notation will also similarly apply to other vectors 
in the proof. 

The core of our proof focuses on the critical case where the dimension $p = 2$ (Section~\ref{sec:p=2}). 
We subsequently extend these results to any dimension $p\ge 2$, which is straightforward
(Section~\ref{sec:general-dimension-p-ge-2}).

\subsection{The case $p = 2$}
\label{sec:p=2}
We establish 
Theorem~\ref{theorem:inconsistency-HSIC} and 
Theorem~\ref{theorem:inconsistency-HSIC-b} for the case $p = 2$.

For our counterexample, we construct a family of distributions $\P_\Delta$ parameterized 
by a vector $\Delta$. Our proof proceeds by showing that there exists a vector $\Delta$ such that if the 
data $(X, Y)$ follows the distribution $\P_\Delta$, then the conclusions of 
Theorem~\ref{theorem:inconsistency-HSIC} 
and Theorem~\ref{theorem:inconsistency-HSIC-b} hold.

\subsubsection{Definition of $\P_\Delta$} 
\label{sec:definition-of-P-Delta}
The distribution $\P_\Delta$ is defined on a product space $\{\pm 1\}^2 \times \{\pm 1\}$, 
with the covariates $X \in \{\pm 1\}^2$ and the response $Y \in \{\pm 1\}$. The parameter 
$\Delta = (\Delta_1, \Delta_2) \in [0, 1]^2$.

Below we specify the details of the construction.  
\begin{itemize}
\item The response $Y \in \{\pm 1\}$ is balanced:  $\P_\Delta(Y = +1) = \P_\Delta(Y = -1) = 1/2$. 
\item The covariate $X \in \{\pm 1\}^2$, with components $X = (X_1, X_2)$, obeys a conditional 
independence given $Y$: $X_1 \perp X_2 \mid Y$.  Moreover, the conditional distribution of 
$X_j$ ($j = 1, 2$) given $Y$ obeys:
\begin{equation*}
	\P_\Delta(X_1 = \pm 1 \mid Y) = \half (1 \pm \Delta_1 Y), 
	~~~~~
	\P_\Delta(X_2 = \pm 1 \mid Y) = \half (1 \pm \Delta_2 Y).
\end{equation*} 
\end{itemize} 
This construction determines the joint distribution of $(X_1, X_2, Y)$ as follows: 
\begin{equation*}
	\P_\Delta(X_1= x_1, X_2 = x_2, Y =y) = \frac{1}{8} (1 + \Delta_1 x_1 y) (1+\Delta_2 x_2y)~~\text{
		for $x_1, x_2, y \in \{\pm 1\}$}. 
\end{equation*}
Intuitively, $\Delta_1$ and $\Delta_2$ measure how much $Y$ depends on $X_1$ and 
$X_2$, respectively. For instance, $\Delta_1 = 0$ indicates that $X_1$ and $Y$ are independent, 
while $\Delta_1 = 1$ means $X_1$ perfectly predicts $Y$. 
Thus, when both $\Delta_1, \Delta_2$ lie within $(0, 1)$, this intuition suggests that both features 
$X_1$ and $X_2$ are needed in order to fully explain the response $Y$. 
Lemma~\ref{lemma:both-X_1-X_2-useful} formally establishes this. 

We use $\E_\Delta$ to denote the expectation under $\P_\Delta$. 
\begin{lemma}
\label{lemma:both-X_1-X_2-useful}
Suppose $\Delta_1\in (0, 1), \Delta_2 \in (0, 1)$. Then the following inequalities hold under 
$\mathcal{L}_2(\P_\Delta)$: 
\begin{equation*}
	\E_\Delta[Y|X] \neq \E_\Delta[Y|X_1], ~~\E_\Delta[Y|X] \neq \E_\Delta[Y|X_2],~~\E_\Delta[Y|X] \neq \E_\Delta[Y].
\end{equation*} 
\end{lemma} 
The proof of Lemma~\ref{lemma:both-X_1-X_2-useful} is based on routine calculations, 
and is deferred to Section~\ref{sec:proof-of-lemma:both-X_1-X_2-useful}.


%
%
%
%

\subsubsection{Evaluation of HSIC}
\label{sec:evaluation-of-HSIC}
We evaluate $\HSIC(\beta \odot X; Y)$ when $(X, Y) \sim \P_\Delta$ for every $\beta \in \R^2$. 

Recall $k_X(x, x') = \phi_X(\norm{x-x'}_2^2)$. For $\beta \in \R^2$, we define the function
\begin{equation}
\label{eqn:expression-of-L-Delta-beta}
\begin{split} 
	L_\Delta(\beta) &= (\phi_X(0)-\phi_X(4\beta_1^2 + 4\beta_2^2)) \cdot (\Delta_1^2+ \Delta_2^2)
			- (\phi_X(4 \beta_1^2) -\phi_X(4\beta_2^2)) \cdot (\Delta_1^2 - \Delta_2^2). 
\end{split} 
\end{equation}
The function $\beta \mapsto L_\Delta(\beta)$ is proportional to the function $\beta \mapsto \HSIC(\beta \odot X; Y)$, differing only by 
a positive, constant scalar. 
See Lemma~\ref{lemma:expression-of-L-Delta-beta} for the formal statement. 

\begin{lemma}
\label{lemma:expression-of-L-Delta-beta}
Let $(X, Y) \sim \P_\Delta$. Then for every $\beta = (\beta_1, \beta_2) \in \R^2$:
\begin{itemize}
\item If $k_Y$ is a function of $yy'$, say $k_Y(y, y) = \phi_Y(yy')$, then 
	\begin{equation*}
		\HSIC(\beta \odot X; Y) = \frac{1}{8} \cdot (\phi_Y(1)-\phi_Y(-1)) \cdot L_\Delta(\beta).
	\end{equation*} 
\item If $k_Y$ is a function of $|y-y'|$, say $k_Y(y, y) = \phi_Y(|y-y'|)$, then 
	\begin{equation*}
		\HSIC(\beta \odot X; Y) =\frac{1}{8} \cdot (\phi_Y(0)-\phi_Y(1)) \cdot L_\Delta(\beta).
	\end{equation*} 
\end{itemize}
\end{lemma} 
The proof of Lemma~\ref{lemma:expression-of-L-Delta-beta} is due to calculations, although 
tricks can be employed to simplify the calculation 
process. For the details of the proof, see Section~\ref{sec:proof-of-lemma:expression-of-L-Delta-beta}.

Lemma~\ref{lemma:expression-of-L-Delta-beta} reveals that the function
$\beta \mapsto \HSIC(\beta\odot X; Y)$ is equivalent to $\beta \mapsto L_\Delta(\beta)$ 
scaled by a positive constant $a$, with $a > 0$ by Assumption~\ref{assumption:kernel-representation-Y}. 
Therefore, maximizing $\beta \mapsto \HSIC(\beta\odot X; Y)$ corresponds directly to maximizing 
$\beta \mapsto L_\Delta(\beta)$ on any given set of constraints on $\beta$.

\subsubsection{Properties of $L_\Delta$}
Below we shall frequently assume $\Delta_1 > \Delta_2 > 0$. In this 
setup, both features $X_1$ and $X_2$ are relevant features, with 
$X_1$ being called the \emph{dominant} feature, and $X_2$ being called the \emph{weaker} feature. 
We shall derive a collection of structural properties of $L_\Delta$ under this setup. 

A quick observation is the symmetry of $L_\Delta$. 
\begin{lemma}[Symmetry]
\label{lemma:symmetry}
For every $\beta \in \R^2$, $L_\Delta(\beta_1, \beta_2) = L_\Delta(|\beta_1|, |\beta_2|)$.
\end{lemma}

This symmetry property enables us to focus our understanding on how $L_\Delta$ behaves 
on 
\begin{equation*}
	\R_+^2 = \{(\beta_1, \beta_2): \beta_1 \ge 0, \beta_2 \ge 0\}.
\end{equation*}  
Lemma~\ref{lemma:monotonically-increasing-in-beta-one} then shows that increasing 
the coefficient $\beta_1$ on $\R_+$ for the dominant feature $X_1$ strictly increases the function 
value $L_\Delta(\beta)$. 
\begin{lemma}[Monotonicity in Dominant Feature]
\label{lemma:monotonically-increasing-in-beta-one}
Let $\Delta_1 > \Delta_2 > 0$. Then for every $\beta_2 \in \R_+$, 
\begin{equation*}
	\beta_1 \mapsto L_\Delta(\beta_1, \beta_2)
\end{equation*}
is strictly monotonically increasing for $\beta_1 \in \R_+$. 
\end{lemma} 

\begin{proof}
The mapping $z \mapsto \phi_X(z)$ is strictly monotonically decreasing on $\R_+$, as
\begin{equation*}
	\phi_X(z) = \int_0^\infty e^{-tz} \mu(dt)
\end{equation*}
for some nonnegative measure $\mu$ not atomic at zero. The strict 
monotonicity of $\beta_1 \mapsto L_\Delta(\beta_1, \beta_2)$ when $\beta_1 \in \R_+$
then follows from the definition of $L_\Delta$ (equation~\eqref{eqn:expression-of-L-Delta-beta}) in 
Section~\ref{sec:evaluation-of-HSIC}.
\end{proof} 

Lemma~\ref{lemma:permutation-increasing} further demonstrates that, on $\beta \in \R_+^2$, increasing
$L_\Delta(\beta)$ can be effectively achieved by allocating higher coefficients to the 
dominant feature $X_1$ rather than the weaker feature $X_2$.
\begin{lemma}
\label{lemma:permutation-increasing}
Let $\Delta_1 > \Delta_2 > 0$. Then for every $\beta = (\beta_1, \beta_2) \in \R_+^2$ with 
$\beta_1 > \beta_2$:  
\begin{equation*}
	L_\Delta(\beta_1, \beta_2) > L_\Delta(\beta_2, \beta_1). 
\end{equation*}
\end{lemma} 

\begin{proof}
Following the definition of $L_\Delta$ (equation~\eqref{eqn:expression-of-L-Delta-beta}), we
deduce that 
\begin{equation*}
	L_\Delta(\beta_1, \beta_2) - L_\Delta(\beta_2, \beta_1) = 
		-2 (\phi_X(4\beta_1^2) - \phi_X(4 \beta_2^2))(\Delta_1^2 - \Delta_2^2).
\end{equation*}
The result then follows since $\phi_X$ is strictly monotonically decreasing. 
\end{proof} 

Lemma~\ref{lemma:reducing-weaker-signal-to-zero} demonstrates a key and somewhat surprising
property of $L_\Delta$ and the dependence measure $\beta \mapsto \HSIC(\beta \odot X; Y)$. 
It demonstrates that if the dominant feature $X_1$ is weighted by a sufficiently large $\beta_1$ that
meets the condition in equation~\eqref{eqn:technical-condition-for-elimination}, 
then setting the weight $\beta_2$ of the weaker feature $X_2$ to zero results in an increase in
$L_\Delta$.  
This implies that the dependence measure $\HSIC(\beta \odot X; Y)$, which is directly proportional to $L_\Delta$ 
as explained in Lemma~\ref{lemma:expression-of-L-Delta-beta}, also increases.

\begin{lemma}[Removing Weaker Signal Increases $L_\Delta$]
\label{lemma:reducing-weaker-signal-to-zero}
Let $\Delta_1 > \Delta_2 > 0$. Suppose $\beta_1$ satisfies
\begin{equation}
\label{eqn:technical-condition-for-elimination}
	 \frac{1}{|\mu|}\int_0^\infty e^{- 4\beta_1^2 t} \mu(dt) < \frac{\Delta_1^2 - \Delta_2^2}{\Delta_1^2 + \Delta_2^2}
\end{equation} 
where $|\mu|= \mu([0, \infty))$ denotes the total measure. 
Then for every $\beta_2 > 0$: 
\begin{equation*}
	L_\Delta(\beta_1, \beta_2) < L_\Delta(\beta_1, 0). 
\end{equation*} 
\end{lemma} 
\begin{proof}
Using the expression of $L_\Delta$ in Lemma~\ref{lemma:expression-of-L-Delta-beta}, we obtain 
\begin{equation*}
	L_\Delta(\beta_1, 0) - L_\Delta(\beta_1, \beta_2) = 
		(\phi_X(4\beta_1^2 + 4\beta_2^2) - \phi_X(4\beta_1^2))(\Delta_1^2 + \Delta_2^2)
			+ (\phi_X(0) - \phi_X(4\beta_2^2))(\Delta_1^2 - \Delta_2^2).
\end{equation*} 
If $\beta_2 > 0$ then $\phi_X(0) > \phi_X(4\beta_2^2)$ because $\phi_X$ is strictly monotonically increasing on $\R_+$. 
Therefore,  
$L_\Delta(\beta_1, \beta_2) < L_\Delta(\beta_1, 0)$ if and only if 
\begin{equation}
\label{eqn:equivalent-condition-to-remove-weaker-signal}
	\frac{\phi_X(4\beta_1^2) - \phi_X(4\beta_1^2 + 4\beta_2^2)}{\phi_X(0) - \phi_X(4\beta_2^2)} 
		< \frac{\Delta_1^2 - \Delta_2^2}{\Delta_1^2 + \Delta_2^2}.
\end{equation} 
Below we show that the condition~\eqref{eqn:equivalent-condition-to-remove-weaker-signal} holds 
under Assumption~\eqref{eqn:technical-condition-for-elimination}. 

To do so, a key step is to show the following upper bound on the ratio: 
\begin{equation}
\label{eqn:technical-bound-on-the-ratio}
	\frac{\phi_X(4\beta_1^2) - \phi_X(4\beta_1^2 + 4\beta_2^2)}{\phi_X(0) - \phi_X(4\beta_2^2)} 
		\le  \frac{1}{|\mu|}\int_0^\infty e^{-4\beta_1^2 t} \mu(dt).
\end{equation} 
Note that 
\begin{equation*}
\begin{split} 
	\phi_X(4\beta_1^2) - \phi_X(4\beta_1^2 + 4\beta_2^2) &= \int_0^\infty e^{-4\beta_1^2 t}
		(1-e^{-4t\beta_2^2 t}) \mu(dt) \\
	\phi_X(0) - \phi_X(4\beta_2^2) &= \int_0^\infty 
		(1-e^{-4t\beta_2^2 t}) \mu(dt)
\end{split}. 
\end{equation*} 
Since the function $t \mapsto e^{-4\beta_1^2 t}$ is monotonically decreasing, and the function 
$t \mapsto 1-e^{-4t\beta_2^2 t}$ is monotonically increasing, Chebyshev's sum inequality 
becomes relevant to bound the ratio.

\begin{lemma}[Chebyshev's sum inequality]
If $f$ and $g$ are real-valued integrable function over $[a, b]$, both nonincreasing. 
Suppose $\nu$ is a nonnegative finite measure on $[a, b]$ with $\nu([a, b])< \infty$. 
Then 
\begin{equation*}
	\int_a^b f(x) g(x) \nu(dx) \ge \frac{1}{\nu([a,b])}\int_a^b f(x)\nu(dx) \int_a^b g(x)\nu(dx). 
\end{equation*}
\end{lemma} 
Specifically, if we choose $f(t) = e^{-4\beta_1^2 t}$ and $g(t) = -(1-e^{-4t\beta_2^2 t})$, 
both functions monotonically decreasing, and apply the Chebyshev's sum inequality to these 
two functions, we obtain
\begin{equation*}
	\phi_X(4\beta_1^2) - \phi_X(4\beta_1^2 + 4\beta_2^2)
		= -\int_0^\infty f(t)g(t)\mu(dt) \le - \frac{1}{|\mu|}\int_0^\infty f(t) \mu(dt) \int_0^\infty g(t) \mu(dt).
\end{equation*} 
and thereby: 
\begin{equation}
\begin{split} 
	\frac{\phi_X(4\beta_1^2) - \phi_X(4\beta_1^2 + 4\beta_2^2)}{\phi_X(0) - \phi_X(4\beta_2^2)} 
		&= \frac{-\int_0^\infty f(t)g(t)\mu(dt)}{-\int_0^\infty g(t)\mu(dt)}  \\
		&\le \frac{1}{|\mu|} \int_0^\infty f(t) \mu(dt)  
		= \frac{1}{|\mu|}\int_0^\infty e^{-4\beta_1^2 t} \mu(dt).
\end{split} 
\end{equation} 
This proves equation~\eqref{eqn:technical-bound-on-the-ratio}. 

Given equation~\eqref{eqn:technical-bound-on-the-ratio} and the 
assumption in equation~\eqref{eqn:technical-condition-for-elimination}, we know that the condition
\eqref{eqn:equivalent-condition-to-remove-weaker-signal} is met. The desired 
conclusion $L_\Delta(\beta_1, \beta_2) < L_\Delta(\beta_1, 0)$ then follows. 
\end{proof}

\subsubsection{Finalizing Arguments}
In below, we complete the proof of Theorem~\ref{theorem:inconsistency-HSIC}
and Theorem~\ref{theorem:inconsistency-HSIC-b} for the case $p = 2$. 

\paragraph{Proof of Theorem~\ref{theorem:inconsistency-HSIC}}
Let us choose $1 > \Delta_1 > \Delta_2 > 0$ in a way such that 
\begin{equation}
\label{eqn:condition-for-feature-selection}
	 \frac{1}{|\mu|}\int_0^\infty e^{- 4 t} \mu(dt) < \frac{\Delta_1^2 - \Delta_2^2}{\Delta_1^2 + \Delta_2^2}.
\end{equation}
This is achievable since the LHS is strictly smaller than $1$, as $\mu$ is not an atomic 
measure at zero. Now we consider the HSIC maximization procedure for the data 
$(X, Y) \sim \P_\Delta$, where $\P_\Delta$ is specified in Section~\ref{sec:definition-of-P-Delta}:  
\begin{equation*}
	S_* = \argmax_{S: S \subseteq \{1, 2\}} \HSIC(X_S; Y).
\end{equation*}
\begin{itemize}
\item ($X_1$ is selected in the maximizer) 
	Since $\Delta_1 > \Delta_2$, by Lemma~\ref{lemma:monotonically-increasing-in-beta-one} we obtain 
	$L_\Delta(1, 0) > L_\Delta(0, 0)$ and~$L_\Delta(1, 1) > L_\Delta(0, 1)$. 
	These two inequalities, by Lemma~\ref{lemma:expression-of-L-Delta-beta}, are equivalent to
	\begin{equation*}
	\begin{split} 
		\HSIC(X_{\{1\}}; Y) &> \HSIC(X_{\emptyset}; Y) \\
		\HSIC(X_{\{1, 2\}}; Y) &> \HSIC(X_{\{2\}}; Y).
	\end{split} 
	\end{equation*}
\item ($X_2$ is not selected if $X_1$ is selected)
	Since condition~\eqref{eqn:condition-for-feature-selection} is assumed, by 
	Lemma~\ref{lemma:reducing-weaker-signal-to-zero} we obtain 
	$L_\Delta(1, 0) > L_\Delta(1, 1)$. This inequality, by Lemma~\ref{lemma:expression-of-L-Delta-beta}, 
	is equivalent to
	\begin{equation*}
	\begin{split} 
		\HSIC(X_{\{1\}}; Y) &> \HSIC(X_{\{1, 2\}}; Y).
	\end{split} 
	\end{equation*}
\end{itemize}
Consequentially, we deduce that $S_* = \{1\}$ maximizes the objective 
$S \mapsto \HSIC(X_S; Y)$ for any $\Delta$ obeying equation~\eqref{eqn:condition-for-feature-selection}
with $1 > \Delta_1 > \Delta_2 > 0$. Nonetheless,
$\E_\Delta[Y|X] \neq \E_\Delta[Y|X_{S_*}]$ by Lemma~\ref{lemma:both-X_1-X_2-useful}. 

\paragraph{Proof of Theorem~\ref{theorem:inconsistency-HSIC-b}}
Fix the norm $\norm{\cdot}_{\ell_q}$ and $0 < r < \infty$. 

Let us define $\Phi(b) = \norm{(b, b)}_{\ell_q} = 2^{1/q}b^{1/q}$ for every $b \in \R_+$. Then $\Phi$ 
is strictly monotonically increasing on $\R_+$, and continuous. Let $b_0 = \Phi^{-1}(r)$. Given this 
$b_0$, we then choose $\Delta_1 > \Delta_2 > 0$ in a way such that 
\begin{equation}
\label{eqn:condition-for-feature-selection-continuous}
	 \frac{1}{|\mu|}\int_0^\infty e^{- 4 b_0^2 t} \mu(dt) < \frac{\Delta_1^2 - \Delta_2^2}{\Delta_1^2 + \Delta_2^2}.
\end{equation}
This is achievable since the LHS is strictly smaller than $1$. 

We consider the continuous version of HSIC maximization procedure for the data 
$(X, Y) \sim \P_\Delta$:  
\begin{equation*}
	S_{*,\cont} = \supp(\beta_*)~~\text{where}~~\beta_* = \argmax_{\beta: \norm{\beta}_{\ell_q} \le r} 
		\HSIC(\beta \odot X; Y)
\end{equation*}
By Lemma~\ref{lemma:expression-of-L-Delta-beta}, $\beta_*$ must also be a maximizer of 
\begin{equation*}
	\beta_* = \argmax_{\beta: \norm{\beta}_{\ell_q} \le r}  L_\Delta(\beta). 
\end{equation*} 
Write $\beta_* = (\beta_{*, 1}, \beta_{*, 2})$. We may assume $\beta_* \in \R_+^2$
by the symmetry property in Lemma~\ref{lemma:symmetry}. Note: 
\begin{itemize}
\item ($X_1$ is selected and $\beta_{*, 1} \ge b_0$) Given that $\Delta_1 > \Delta_2$, 
	Lemma~\ref{lemma:monotonically-increasing-in-beta-one} implies that 
	$L_\Delta(\beta_1, \beta_2)$  monotonically increases strictly with respect to 
	$\beta_1$ on $\R_+$. 
	Thus, at the maximizer $\beta_* \in \R_+^2$, the constraint $\norm{\beta_*}_{\ell_q} = r$ must hold,
	since otherwise we can strictly increase the objective $L_\Delta$ by increasing $\beta_1$. 
	Additionally, since $\Delta_1 > \Delta_2$, Lemma~\ref{lemma:permutation-increasing} 
	further implies that $\beta_{*, 1} \ge \beta_{*, 2}$ at the maximizer. 
	
	Hence, we obtain that $r = \norm{(\beta_{*,1}, \beta_{*, 2})}_{\ell_q} \le  \norm{(\beta_{*,1}, \beta_{*, 1})}_{\ell_q}
		= \Phi(\beta_{*,1})$. Since $\Phi$ is strictly monotonically increasing and continuous, 
	we obtain that $\beta_{*, 1} \ge \Phi^{-1}(r) = b_0$.
\item ($X_2$ is not selected, i.e., $\beta_{*, 2} = 0$) Since 
	condition~\eqref{eqn:condition-for-feature-selection-continuous} is assumed, 
	and $\beta_{*, 1} \ge b_0$, we derive 
	\begin{equation*}
		 \frac{1}{|\mu|}\int_0^\infty e^{- 4 \beta_{*, 1}^2 t} \mu(dt) 
		 	< \frac{\Delta_1^2 - \Delta_2^2}{\Delta_1^2 + \Delta_2^2}.
	\end{equation*} 
	Lemma~\ref{lemma:reducing-weaker-signal-to-zero} then implies that  
	$L_\Delta(\beta_{*, 1}, 0) > L_\Delta(\beta_{*, 1}, \beta_{2})$ for every 
	$\beta_{2} \neq 0$. Since $\beta_*$ is the maximizer, this further implies that 
	$\beta_{*, 2} = 0$ must hold. 
\end{itemize} 
Consequentially, we deduce that $S_{*,\cont} = \supp(\beta_*) = \{1\}$. However, 
$\E_\Delta[Y|X] \neq \E_\Delta[Y|X_{S_*, \cont}]$ by Lemma~\ref{lemma:both-X_1-X_2-useful}.

\subsection{General dimension $p \ge 2$}
\label{sec:general-dimension-p-ge-2}
Previously in Section~\ref{sec:p=2}, we have already established 
Theorem~\ref{theorem:inconsistency-HSIC} and 
Theorem~\ref{theorem:inconsistency-HSIC-b} when the dimension $p = 2$. Specifically, 
we have constructed a distribution of $(X, Y)$ supported on $\R^2 \times \R$ such that 
the conclusions of Theorem~\ref{theorem:inconsistency-HSIC} and 
Theorem~\ref{theorem:inconsistency-HSIC-b} hold.

Now consider an arbitrary dimension $p\ge 2$. Suppose we are given two kernels $k_X$ and 
$k_Y$ on $\R^p$ and $\R$, respectively, which satisfy 
Assumptions~\ref{assumption:kernel-representation} and~\ref{assumption:kernel-representation-Y}. 
In particular, this means that $k_X(x, x') = \phi_X(\norm{x-x'}_2^2)$ for every $x, x' \in \R^p$ 
with some function $\phi_X$ obeying the integral representation in Assumption~\ref{assumption:kernel-representation}.
This kernel $k_X$ on $\R^p$ then induces a kernel 
$\bar{k}_X$ on $\R^2$ as $\bar{k}_X(x, x') = \phi_X(\norm{x-x'}_2^2)$ for all
$x, x' \in \R^2$. Because we have shown the validity of Theorem~\ref{theorem:inconsistency-HSIC} and 
Theorem~\ref{theorem:inconsistency-HSIC-b} for $p=2$, we can conclude that for the kernels $\bar{k}_X$ and $k_Y$,
there exists a probability distribution $\bar{\P}$ for $(\bar{X}, \bar{Y})$ 
over $\R^2 \times \R$, where $\bar{X} \in \R^2$ and $\bar{Y} \in \R$, such 
that the conclusions of 
Theorem~\ref{theorem:inconsistency-HSIC} and Theorem~\ref{theorem:inconsistency-HSIC-b}
hold under this distribution $\bar{\P}$. 

Now we use $\bar{\P}$ to construct a distribution $\P$ on $\R^p \times \R$ such that 
Theorem~\ref{theorem:inconsistency-HSIC} and Theorem~\ref{theorem:inconsistency-HSIC-b} 
hold for the kernels $k_X$ and $k_Y$ on $\R^p$ and $\R$, respectively, under this 
distribution $\P$. Given random variables $(\bar{X}, \bar{Y})$ on $\R^2 \times \R$ following distribution $\bar{\P}$, we 
construct random variables $(X, Y)$ on $\R^p \times \R$ as follows: $X_1 = \bar{X}_1$, 
$X_2 = \bar{X}_2$, $X_3 = X_4 = \ldots = X_p = 0$ and $Y = \bar{Y}$. This construction essentially 
extends a two-dimensional random vector $\bar{X}$ into a $p$-dimensional random vector $X$
by simply padding with zeros, which do not affect the evaluation of the kernel distances. 
Given that the conclusions of Theorem~\ref{theorem:inconsistency-HSIC} and 
Theorem~\ref{theorem:inconsistency-HSIC-b} hold under the distribution $\bar{\P}$, it then follows that the 
conclusions of Theorem~\ref{theorem:inconsistency-HSIC} and Theorem~\ref{theorem:inconsistency-HSIC-b} 
extend to the distribution $\P$ we have constructed.


\appendix

\section{Proofs of Technical Lemma} 
\subsection{Proof of Lemma~\ref{lemma:both-X_1-X_2-useful}}
\label{sec:proof-of-lemma:both-X_1-X_2-useful}
The proof is based on calculations. We recognize that for $(x_1, x_2) \in \{\pm 1\}^2$: 
\begin{equation*}
	\P_\Delta(X_1 = x_1, X_2 = x_2) = \sum_{y: y \in \{\pm 1\}} \P_\Delta(X_1= x_1, X_2 = x_2, Y =y)
		= \frac{1}{4} (1+\Delta_1\Delta_2 x_1 x_2). 
\end{equation*} 
Thereby, we derive for $(x_1, x_2) \in \{\pm 1\}^2$ and $y \in \{\pm 1\}$: 
\begin{equation*}
	\P_\Delta(Y = y|X_1 = x_1, X_2 = x_2) = 
		 \half + \half \cdot \frac{
			y(\Delta_1 x_1 +  \Delta_2 x_2)}{1+\Delta_1 \Delta_2 x_1 x_2},
\end{equation*} 
which immediately gives the conditional expectation of $Y$ given $X$: 
\begin{equation}
\label{eqn:full-conditional-expectation} 
	\E_\Delta[Y|X] = \frac{(\Delta_1 X_1 + \Delta_2 X_2)}{1+\Delta_1 \Delta_2 X_1 X_2}.
\end{equation} 

Similarly, by Bayes rule, we obtain for $(x_1, x_2) \in \{\pm 1\}^2$ and $y \in \{\pm 1\}$: 
\begin{equation*}
	\P_\Delta(Y = y|X_1 = x_1) = \half (1+\Delta_1 x_1 y),~~~
	\P_\Delta(Y = y|X_2 = x_2) = \half (1+\Delta_2 x_2 y).
\end{equation*} 
As a consequence, we deduce: 
\begin{equation}
\label{eqn:partial-conditional-expectation}
	\E_\Delta[Y|X_1] = \Delta_1 X_1,~~~\E_\Delta[Y|X_2] = \Delta_1 X_2,~~~\E_\Delta[Y]= 0.
\end{equation} 

Comparing the expressions in equations~\eqref{eqn:full-conditional-expectation} and 
\eqref{eqn:partial-conditional-expectation}, we reach Lemma~\ref{lemma:both-X_1-X_2-useful}
as desired.

\subsection{Proof of Lemma~\ref{lemma:expression-of-L-Delta-beta}}
\label{sec:proof-of-lemma:expression-of-L-Delta-beta}
The proof can be done by mechanical calculations. Here we will mainly present the tricks 
we use to simplify the calculation process. Notably, because the proofs for the two cases of $k_Y$ are similar, below we only present 
the proof for the first case where $k_Y(y, y') = \phi_Y(yy')$. 

Recall the definition: 
\begin{equation}
\label{eqn:recall-definition-HSIC}
\begin{split} 
	\HSIC(\beta \odot X, Y) &= \E_\Delta[k_{X}(\beta \odot X, \beta \odot X') \cdot  k_Y(Y, Y')] + 
		 \E_\Delta[k_{X}(\beta \odot X, \beta \odot X')] \cdot  \E_\Delta[k_Y(Y, Y')] \\
		&~~~~ - 2 \E_\Delta \left[ \E_\Delta[k_{X}(\beta \odot X, \beta \odot X')|X'] \cdot  
			\E_\Delta[k_Y(Y, Y')|Y']\right].
\end{split} 
\end{equation}
where $(X, Y), (X', Y')$ are independent copies from $\P_\Delta$. 

Define $\tilde{k}_Y(y, y') = k_Y(y, y') - c$, with $c \in \R$ to be determined later. Then with any $c \in \R$, 
the value of $\HSIC(\beta \odot X, Y)$ does not change if we replace $k_Y$ in the RHS of 
equation~\eqref{eqn:recall-definition-HSIC} by $\tilde{k}_Y$, i.e.,
\begin{equation}
\label{eqn:recall-definition-HSIC-two}
\begin{split} 
	\HSIC(\beta \odot X, Y) &=  \E_\Delta[k_{X}(\beta \odot X, \beta \odot X') \cdot  \tilde{k}_Y(Y, Y')] + 
		 \E_\Delta[k_{X}(\beta \odot X, \beta \odot X')] \cdot  \E_\Delta[\tilde{k}_Y(Y, Y')] \\
		&~~~~ - 2 \E_\Delta \left[ \E_\Delta[k_{X}(\beta \odot X, \beta \odot X')|X'] \cdot  
			\E_\Delta[\tilde{k}_Y(Y, Y')|Y']\right].
\end{split} 
\end{equation}
Here comes the simplification. We can choose $c \in \R$ such that $ \E_\Delta[\tilde{k}_Y(Y, Y')|Y'] = 0$. 
Note first: 
\begin{equation*}
	 \E_\Delta[\tilde{k}_Y(Y, Y')|Y'] = \E_\Delta[\phi_Y(YY') \mid Y] - c =  \half (\phi_Y(1) + \phi_Y(-1)) - c
\end{equation*} 
where the last identity uses the fact that $Y, Y'$ are independent and uniformly sampled from $\pm 1$.
Thus, if we choose $c = \half (\phi_Y(1) + \phi_Y(-1))$, then two expectations will vanish: 
\begin{equation*}
	\E_\Delta[\tilde{k}_Y(Y, Y')|Y'] = 0 = \E_\Delta[\tilde{k}_Y(Y, Y')]
\end{equation*} 
By substituting this into equation~\eqref{eqn:recall-definition-HSIC-two}, we obtain the identity 
\begin{equation}
\label{eqn:HSIC-simplified}
\begin{split} 
	\HSIC(\beta \odot X, Y) &=  \E_\Delta[k_{X}(\beta \odot X, \beta \odot X') \cdot  \tilde{k}_Y(Y, Y')] \\
		&=  \E_\Delta[k_{X}(\beta \odot X, \beta \odot X') \cdot  YY'] \cdot \half (\phi_Y(1)-\phi_Y(-1)).
\end{split} 
\end{equation} 
where the second identity holds because $\tilde{k}_Y(y, y') = \phi_Y(yy')-\half (\phi_Y(1) + \phi_Y(-1))= \half (\phi_Y(1)-\phi_Y(-1)) yy'$ for 
every pair $y, y' \in \{\pm 1\}$. This gives a much simpler expression than the original one in 
equation~\eqref{eqn:recall-definition-HSIC}.

To further evaluate the RHS of equation~\eqref{eqn:HSIC-simplified}, we compute the conditional expectation: 
\begin{equation*}
\begin{split} 
	V(Y, Y') &= \E_\Delta \left[k_{X}(\beta \odot X, \beta \odot X') \mid Y, Y'\right] \\
		&=\E_\Delta \left[\phi_X(\normsmall{\beta \odot (X-X')}_2^2) \mid Y, Y'\right].
\end{split} 
\end{equation*}
Since $X_1, X_1'$ can only take values in $\pm 1$, $|X_1 - X_1'|$ can only be 
$0$ or $2$. Thus $|X_1 - X_1'| = 2 \mathbf{1}_{X_1 \neq X_1'}$. 
The same applies to $|X_2 - X_2'|$. As a result, we obtain the expression: 
\begin{equation*}
	V(Y, Y') = \E_\Delta\left[\phi_X(4\beta_1^2\mathbf{1}_{X_1 \neq X_1'} 
		+ 4\beta_2^2 \mathbf{1}_{X_2 \neq X_2'}) \mid Y, Y'\right].
\end{equation*}
Given $Y, Y'$,  the random variables $X_1, X_2, X_1', X_2'$ are mutually independent 
by definition of $\P_\Delta$, and thus $\mathbf{1}_{X_1 \neq X_1'}$ is independent of 
$\mathbf{1}_{X_2 \neq X_2'}$. Furthermore, we have 
\begin{equation*}
\begin{split} 
\P_\Delta(X_1 \neq X_1' \mid Y, Y') &= \half (1-\Delta_1^2 YY'),~~~
\P_\Delta(X_2 \neq X_2' \mid Y, Y') = \half (1-\Delta_2^2 YY').
\end{split}
\end{equation*}
As a result, we obtain 
\begin{equation*}
\begin{split}
V(Y, Y') &= \frac{1}{4} \Big[\phi_X(4\beta_1^2 + 4\beta_2^2) (1-\Delta_1^2 YY')(1-\Delta_2^2 YY')
	+\phi_X(4\beta_1^2) (1-\Delta_1^2 YY')(1+\Delta_2^2 YY') \\
	&~~~~~~~~\phi_X(4\beta_2^2) (1+\Delta_1^2 YY')(1-\Delta_2^2 YY')+ \phi_X(0) 
		(1+\Delta_1^2 YY')(1+\Delta_2^2 YY') \Big].
\end{split} 
\end{equation*}

Finally, we go back to equation~\eqref{eqn:HSIC-simplified}. By recognizing the fact that 
\begin{equation*}
	\HSIC(\beta \odot X, Y)  = \E_\Delta[YY' V(Y, Y')] \cdot \half (\phi_Y(1)-\phi_Y(-1)),
\end{equation*} 
and the expectations 
\begin{equation*}
\begin{split} 
	\E_\Delta[(1-\Delta_1^2 YY')(1-\Delta_2^2 YY')YY'] &= - \Delta_1^2 - \Delta_2^2 \\
	\E_\Delta[(1-\Delta_1^2 YY')(1+\Delta_2^2 YY')YY'] &= - \Delta_1^2 + \Delta_2^2 \\
	\E_\Delta[(1+\Delta_1^2 YY')(1-\Delta_2^2 YY')YY'] &=  +\Delta_1^2 - \Delta_2^2 \\
	\E_\Delta[(1+\Delta_1^2 YY')(1+\Delta_2^2 YY')YY'] &=  +\Delta_1^2 + \Delta_2^2 \\
\end{split} 
\end{equation*}
we deduce the expression: 
\begin{equation*}
\begin{split} 
	\HSIC(\beta \odot X, Y) &= \frac{1}{8} \cdot 
		\Big[\phi_X(4\beta_1^2 + 4\beta_2^2)(-\Delta_1^2- \Delta_2^2) 
			+ \phi_X(4\beta_1^2) (-\Delta_1^2 + \Delta_2^2) \\
		&~~~~~~~~~~~~~~~~	+ \phi_X(4\beta_2^2) (-\Delta_2^2 + \Delta_1^2) 
			+ \phi_X(0) (+\Delta_1^2+\Delta_2^2)\Big] \cdot (\phi_Y(1)-\phi_Y(-1)) \\
		&= \frac{1}{8} \cdot (\phi_Y(1)-\phi_Y(-1)) \cdot L_\Delta(\beta). 
\end{split} 
\end{equation*} 
This completes the proof.

\bibliographystyle{amsalpha}
\bibliography{bib}

\appendix
\newpage

\end{document}